\documentclass{article}
\usepackage{ijcai26}
\usepackage{amsfonts}
\usepackage{natbib}
\usepackage{amsmath}
\usepackage{multirow}
\usepackage{graphicx}
\usepackage{natbib}
\usepackage{booktabs}
\usepackage{threeparttable}
\usepackage{amssymb} 

\newcommand{\xmark}{\text{\sffamily\bfseries X}} 
\newcommand{\cmark}{\checkmark} 

\usepackage{times}
\usepackage{soul}
\usepackage{url}
\usepackage[hidelinks]{hyperref}
\usepackage[utf8]{inputenc}
\usepackage[small]{caption}
\usepackage{graphicx}
\usepackage{amsmath}
\usepackage{amsthm}
\usepackage{booktabs}
\usepackage{algorithm}
\usepackage{algorithmic}
\usepackage[switch]{lineno}

\newtheorem{theorem}{Theorem}
\newtheorem{assumption}{Assumption}
\newtheorem{Assumption}{Assumption}

\newtheorem{definition}{Definition}
\newtheorem{Definition}{Definition}
\newtheorem{Lemma}{Lemma}
\newtheorem{Theorem}{Theorem}

\title{Privacy Preserving Reinforcement Learning with One-Sided Feedback} 

\author{
Lin Cong$^{1,3}$\thanks{Authors are listed in alphabetical order.}
\and
Guangyan Gan$^1$\footnotemark[1]
\and
Hanzhang Qin$^{2}$\footnotemark[1]
\And
Zhenzhen Yan$^1$\footnotemark[1]\\
\affiliations
$^1$Nanyang Technological University\\
$^2$National University of Singapore\\
$^3$ Cornell SC Johnson College of Business\\
\emails
will.cong@ntu.edu.sg,
GUANGYAN001@e.ntu.edu.sg,
hzqin@nus.edu.sg, yanzz@ntu.edu.sg
}

\begin{document}

\maketitle

\begin{abstract}
 We study reinforcement learning (RL) in multi-dimensional continuous state and action spaces with one-sided feedback, where the agent receives partial observations of the state and obtains reward information for only a subset of the state-action space at each time step. This setting introduces substantial challenges in both learning efficiency and privacy preservation. To address these challenges, we propose \texttt{POOL}, a novel privacy-preserving RL algorithm. We conduct a comprehensive theoretical analysis of \texttt{POOL}, deriving a sample complexity bound of $\tilde{O}((1+E_\rho) H^3 \alpha^{-2})$, which matches the known lower bounds for non-private RL. Here, $E_\rho$  denotes the privacy parameter, $H$ is the time horizon,  and $\alpha$ is optimality-gap parameter.
Our findings show that it is possible to enforce strong privacy guarantees while maintaining high learning efficiency, marking a significant step toward practical, privacy-aware RL in multi-dimensional environments with one-sided feedback.
\end{abstract}

\section{Introduction}

Reinforcement learning (RL) with one-sided feedback refers to a learning setting within a Markov Decision Process (MDP) where the agent receives feedback (e.g., rewards or next-state observations) only for a restricted subset of state-action pairs. Specifically, feedback is observable only when the state-action pair falls within a predefined observable boundary, while no information is revealed otherwise. This partial observability introduces structural challenges in learning transition dynamics and estimating value functions, distinguishing it from traditional RL settings with full feedback. These challenges make standard RL techniques less effective, motivating the study of RL methods under one-sided feedback.  

In many real-world applications, online data collection is infeasible or expensive \citep{levine2020offline}, which motivates the use of offline RL. The objective of offline RL is to learn an optimal policy from pre-collected datasets without further interactions with the environment. This setting is particularly relevant in domains such as robotic control, clinical trials, and education \citep{zhou2024federated,rengarajan2024federated,nie2025toward}. Consequently, in this paper, we focus on offline RL under one-sided feedback.  

While offline RL has been extensively studied, research on RL with one-sided feedback remains limited and predominantly confined to one-dimensional settings~\citep{qin2023sailing,gong2023bandits,gong2020provably}. Moreover, many theoretical studies focus on tabular MDPs with finite, one-dimensional state and action spaces. In contrast, real-world applications such as marketing, autonomous systems, and healthcare involve \emph{continuous} and \emph{multi-dimensional} state-action spaces and often operate under partial or censored feedback. Bridging this gap requires scalable RL algorithms capable of handling complex, multi-dimensional environments with one-sided feedback.  

Furthermore, many applications involve sensitive data, making \emph{privacy} a critical concern~\citep{qiao2023near, qiao2024offline}. For example, in healthcare decision-making, RL algorithms are often used to optimize sequential treatment policies based on patient data~\citep{raghu2017continuous}. Without proper safeguards, such data usage can lead to privacy breaches and regulatory noncompliance. Existing research on differentially private RL has largely been limited to single-dimensional, tabular RL with full information on state–action pairs. Our work addresses this dual challenge: one-sided feedback in multi-dimensional continuous RL while ensuring differential privacy.  

We propose \texttt{POOL} (\underline{P}rivacy-\underline{O}riented \underline{O}ne-Sided \underline{L}earning), a novel privacy-preserving RL algorithm for multi-dimensional continuous MDPs with one-sided feedback. Prior works have either focused on privacy in tabular RL or on one-sided feedback without privacy considerations~\citep{qiao2023near,qiao2024offline,garcelon2021local,hossain2023hiding,qin2023sailing}. To the best of our knowledge, \texttt{POOL} is the first algorithm to jointly address both challenges in continuous, multi-dimensional environments.  

\noindent\textbf{Technical Innovations.}  
Extending private RL frameworks from tabular full-feedback settings~\citep{qiao2024offline} to continuous, multi-dimensional environments with one-sided feedback introduces significant challenges. These arise from the high computational complexity of infinite state and action spaces and the systematic bias induced by partial observability. We tackle these challenges via a carefully designed \emph{partial discretization strategy} and \emph{multi-dimensional piecewise-linear approximation} techniques. Together, these innovations form the first scalable, privacy-preserving RL approach with provable guarantees for continuous, multi-dimensional settings under one-sided feedback. Table~\ref{table:comparison} compares our results with existing work, highlighting improvements in dimensionality, privacy, and feedback handling.  

\begin{table*}[ht]
\centering
\begin{threeparttable}
\begin{tabular}{lcccccc}
\toprule
Literature & Finite/Continuous & Dimension & Private  &One-Sided & Sample Complexity \\
\midrule
\cite{qin2023sailing} & Continuous & One & \xmark& \cmark & $\tilde{O}\left( H^{3}\alpha^{-2}\right)$ \\ 
\cite{sidford2018near} & Finite & One & \xmark & \xmark  & $\Omega\left(H^{-3} \alpha^{-2}|\mathcal{S}| |\mathcal{A}| \right)$ \\ 
\cite{qiao2024offline} & Finite & One & \cmark  & \xmark&$\tilde{O}\left((1+E_\rho) H^{3}\alpha^{-2}\right)$ \\
\textbf{This paper}& Continuous & Multi & \cmark & \cmark & $\tilde{O}\left((1+E_\rho) H^{3}\alpha^{-2}\right)$ \\
\bottomrule
\end{tabular}
\caption{Comparison of our results to existing work. Here, $H$ denotes the episode length, $\alpha$ is the optimality gap, $E_\rho$ is a privacy-related parameter whose precise definition will be provided in Theorem~\ref{th2}, and $|\mathcal{S}|$ and $|\mathcal{A}|$ represent the cardinalities of the state and action spaces, respectively.}
\label{table:comparison}
\end{threeparttable}
\end{table*}

\noindent\textbf{Our Contributions.} Building on these innovations, our key contributions are summarized as follows:

\begin{itemize}
    \item \textbf{Problem Formulation:} We present the \emph{first} privacy-preserving RL framework for multi-dimensional continuous MDPs under one-sided feedback, addressing a previously unexplored gap in the literature.  
    
    \item \textbf{Algorithmic Innovation:} We introduce \texttt{POOL}, a scalable privacy-preserving RL algorithm that integrates partial discretization with multi-dimensional piecewise-linear approximation. This design enables efficient learning in continuous, multi-dimensional spaces while maintaining strong privacy guarantees and outperforming existing approaches under one-sided feedback.  

    \item \textbf{Theoretical Guarantees:} We prove that \texttt{POOL} satisfies $\rho$-zero-concentrated differential privacy ($\rho$-zCDP) (Definition~\ref{def: zCDP}) and establish a sample complexity upper bound of $\tilde{O}\left((1+E_\rho) H^{3}\alpha^{-2}\right)$, which matches the information-theoretic lower bound while ensuring rigorous privacy protection (Theorem~\ref{th2}). This provides formal assurances on both privacy and learning performance.  

    \item \textbf{Empirical Validation:} We conduct extensive experiments in the inventory control problem to demonstrate that \texttt{POOL} not only preserves privacy but also achieves superior performance compared to existing baselines under one-sided feedback.
\end{itemize}

\noindent\textbf{Related Work.}  
Recent years have witnessed significant advancements in privacy-preserving RL~\citep{vietri2020private, shariff2018differentially, guha2013nearly}. \citet{qiao2023near} introduced near-optimal private policy learning for offline RL, while \citet{qiao2024offline} developed differentially private exploration strategies that achieve sublinear regret in online RL. However, these methods are restricted to discrete settings under full feedback, leaving the more complex continuous environments with one-sided feedback largely unexplored.

In parallel, a substantial body of work has focused on RL with continuous state and action spaces, which are crucial in various control applications~\citep{munos1997reinforcement, jia2022policy, zhao2024policy, wang2020reinforcement, antos2007fitted, carrara2019budgeted}. Discretization and function approximation methods~\citep{qin2023sailing, gong2020provably} have shown promise in these settings. However, these advances either operate in restrictive one-dimensional environments or fail to consider privacy concerns~\citep{qin2023sailing, gong2020provably}. To date, no method simultaneously addresses privacy, multi-dimensionality, and continuous control.

RL with one-sided feedback has gained substantial attention due to its practical relevance in a wide range of problems \citep{gong2020provably, zhao2019stochastic, yuan2021marrying}. This includes complex scenarios like second-price auction design \citep{haoyu2020online}, inventory control in lost-sales settings \citep{qin2023sailing, gong2023bandits}, and dynamic product pricing \citep{cohen2020feature}. For example, \citet{feng2018learning} explores learning in repeated auction environments where bidders do not know their valuation for the items being auctioned, leveraging partial feedback structures. \citet{zhao2019stochastic} addresses one-sided  information feedback in bandit problems, while \citet{lobel2018multidimensional} studies multidimensional search problems under one-sided feedback. \citet{gong2020provably} introduces an efficient Q-learning framework for one-sided feedback. However, these works focus primarily on one-dimensional settings and neglect privacy concerns. To the best of our knowledge, we are the first to study multi-dimensional continuous RL with one-sided feedback while incorporating rigorous privacy guarantees.

\textbf{This paper closes a crucial gap.} To the best of our knowledge, this paper is the first to develop scalable privacy-preserving RL algorithms for multi-dimensional continuous MDPs with one-sided feedback. Our framework introduces a scalable partial discretization strategy and a piecewise-linear approximation scheme, enabling rigorous privacy guarantees in complex environments that were previously considered intractable.

\textbf{Notation.} Throughout this paper, we denote the $\ell_p$ norm of a vector $\boldsymbol{a}$ by $\|\boldsymbol{a}\|_p$. For any $H \in \mathbb{N}_+$, let $[H]$ denote the set $\{1, 2, \dots, H\}$. The natural logarithm is denoted by $\log(\cdot)$. Let $\mathbf{I} = [1, 1, \dots, 1] \in \mathbb{R}^{w + d}$ denote the all-one vector of dimension $w + d$. We use the standard asymptotic notations $O(\cdot)$ and $\Omega(\cdot)$ as defined in~\cite{cormen2022introduction}. When logarithmic factors are omitted, we use $\tilde{O}(\cdot)$ and $\widetilde{\Omega}(\cdot)$, respectively.


\section{Preliminaries}
\label{sec:problem}

We first outline the framework of multi-dimensional RL with one-sided feedback in Section~\ref{sec:RL}, followed by a brief overview of differential privacy in Section~\ref{sec:Preliminaries on differential privacy}.

\subsection{Multi-dimensional RL with One-Sided Feedback}
\label{sec:RL}

We consider a finite-horizon episodic MDP with continuous, multi-dimensional state and action spaces, specified by the tuple  
\(\left(\mathcal{S}, \mathcal{A}, H, \left\{P_h\right\}_{h=1}^H, \left\{r_h\right\}_{h=1}^H, d_1\right),\)
where $\mathcal{S} = [0,1]^w$ denotes the $w$-dimensional state space, $\mathcal{A} = [0,1]^d$ denotes the $d$-dimensional action space, and $H$ is the episode horizon.  In contrast to classical tabular MDPs with discrete state-action spaces, we consider a setting where both states and actions lie in continuous spaces. Such settings frequently arise in real-world applications. For instance, in inventory management, where the inventory level represents the state and the order quantity serves as the action, both of which are naturally continuous. Other examples include energy management systems, where battery levels and power output represent the state and action, respectively; financial portfolio optimization, with asset prices as states and allocation proportions as actions; and auction design, where the bidder value distributions form the state and the reserve price serves as the action.
At each time step $h \in [H]$, the (potentially non-stationary) transition kernel $P_h: \mathcal{S} \times \mathcal{A} \times \mathcal{S} \to [0,1]$ maps each state-action pair $(s_h,a_h)$ to a probability distribution $P_h(\cdot \mid s_h, a_h)$ over next states. 
The reward function $r_h: \mathcal{S} \times \mathcal{A} \to [0,1]$ is assumed to be known.\footnote{This is a standard assumption in control and offline RL settings, where reward uncertainty is negligible compared to transition uncertainty; see \cite{qiao2024offline}.} The initial state distribution is denoted by $d_1$.
A policy $\pi = (\pi_1, \ldots, \pi_H)$ consists of stochastic decision rules, where each $\pi_h: \mathcal{S} \to \Delta(\mathcal{A})$ maps a state to a distribution over actions at step $h$. At each step, the agent observes $s_h$, samples $a_h \sim \pi_h(\cdot \mid s_h)$, receives reward $r_h(s_h, a_h)$, and transitions to $s_{h+1} \sim P_h(\cdot \mid s_h, a_h)$. The episode terminates after $H$ steps.  
A trajectory $\left(s_1, a_1, r_1, \ldots, s_H, a_H, r_H, s_{H+1}\right)$ is generated according to: $s_1 \sim d_1$, $a_h \sim \pi_h(\cdot \mid s_h)$, $r_h = r_h(s_h, a_h)$, and $s_{h+1} \sim P_h(\cdot \mid s_h, a_h)$ for each $h \in [H]$.
For notational clarity, we omit explicit arguments when there is no risk of ambiguity. For example, we write $P_h$ instead of $P_h(s_{h+1} \mid s_h, a_h)$, and similarly simplify other functions such as $V_h^{\pi}$ for $V_h^{\pi}(s_h)$ and $Q_h^\pi=Q_h^\pi(s_h, a_h)$, when the context is clear.
\paragraph{Value Functions and Bellman Equations.} 
The value and Q-value functions under policy $\pi$ are defined by:
\(V_h^\pi(s) = \mathbb{E}_\pi\left[\sum_{t=h}^H r_t \mid s_h = s\right], \quad Q_h^\pi(s, a) = \mathbb{E}_\pi\left[\sum_{t=h}^H r_t \mid s_h = s, a_h = a\right], \)
for all $h \in [H]$, $s \in \mathcal{S}$, and $a \in \mathcal{A}$. 
We use the Bellman equations to characterize the system dynamics:
\(Q_h^\pi = r_h + P_h V_{h+1}^\pi\), where \(V_h^\pi = \mathbb{E}_{a \sim \pi_h}\left[Q_h^\pi\right].\)
The value and Q-functions under the optimal policy are given by:
\(Q_h^\star = r_h + P_h V_{h+1}^\star, \quad V_h^\star = \max_{a} Q_h^\star(\cdot, a).\)
The performance metric of each policy $\pi$ is defined by 
\(v^\pi := \mathbb{E}_{d_1}\left[V_1^\pi\right] = \mathbb{E}_{\pi, d_1}\left[\sum_{t=1}^H r_t\right].\) 
We consider an offline RL setting and employ a backward induction approach to solve the optimization problem, initializing with \( V_{H+1}^{\pi} = 0 \) for any $\pi$.

\paragraph{Full-Feedback.}
Under the full-feedback setting, both rewards and transitions for \emph{all} state–action pairs are fully observable, making them entirely learnable.

\paragraph{One-Sided Feedback.}
We first introduce observable boundaries, denoted by $\boldsymbol{\lambda}_h = [\lambda_h^1, \lambda_h^2, \ldots, \lambda_h^{w+d}]$ to capture the feedback structure. This notion is motivated by \cite{qin2023sailing}, which focuses on one-dimensional RL without addressing privacy guarantee.

\begin{definition}[Observable Boundary Comparison]
For any state-action pair $b = (s,a) \in \mathcal{S} \times \mathcal{A}$ and observable boundary $\boldsymbol\lambda = [\lambda^1, \ldots, \lambda^{w+d}]$, we define: (1) $b < \boldsymbol\lambda$ if $b^i < \lambda^i$ for all $i \in [w+d]$; (2) $b \geq \boldsymbol\lambda$ if there exists $i \in [w+d]$ such that $b^i \geq \lambda^i$.
\end{definition}

We now formally introduce one-sided feedback:

\begin{definition}[One-Sided Feedback]
\label{OSF}
At period $h$, the agent observes rewards and transitions only for state-action pairs lying on \emph{one side} of the observable boundary $\boldsymbol{\lambda}_h$.  
Specifically: (1) Under lower one-sided feedback, rewards and transitions are revealed for state-action pairs satisfying $b_h = (s_h, a_h) < \boldsymbol{\lambda}_h$; (2) Under higher one-sided feedback, rewards and transitions are revealed for state-action pairs satisfying $b_h > \boldsymbol{\lambda}_h$.
No feedback is available for state-action pairs on the opposite side of the boundary.
\end{definition}

One-sided feedback naturally arises in applications such as inventory control with lost sales. A concrete example is provided in Appendix.  
Our techniques apply symmetrically to lower and higher one-sided feedback. We focus on \textbf{lower one-sided feedback} without loss of generality.
\paragraph{Objective.}
Consider a dataset $\mathcal{D} = \left\{(s_h^\tau, a_h^\tau, r_h^\tau, s_{h+1}^\tau)\right\}_{\tau \in [n]}^{h \in [H]}$ collected under an unknown behavior policy $\mu$.
By \emph{unknown behavior policy}, we mean that the data were generated by some fixed, possibly stochastic, policy $\mu$, whose form or parameters are not accessible to the learner.  This setting is standard in offline RL, where interaction with the environment is prohibited and the only available information is the static dataset $\mathcal{D}$. The goal of offline reinforcement learning is to find a policy $\pi_{\text{alg}}$ such that
\(v^\star - v^{\pi_{\text{alg}}} \leq \alpha\)
for a target accuracy level $\alpha > 0$, while  ensuring \emph{differential privacy} (see Section~\ref{sec:Preliminaries on differential privacy}).

\subsection{Differential Privacy}
\label{sec:Preliminaries on differential privacy}

Differential privacy (DP) provides a principled framework for ensuring privacy in learning algorithms.  
The standard DP definition \cite{dwork2006differential} is as follows:

\begin{definition}[$(\epsilon, \delta)$-DP]
\label{Def:DP}
A randomized mechanism $\mathcal{M}$ satisfies $(\epsilon, \delta)$-DP if for any pair of neighboring datasets $\mathcal{U}$ and $\mathcal{U}'$, which differ by exactly one data point, and for any measurable event $E$:
\(\mathbb{P}[\mathcal{M}(\mathcal{U}) \in E] \leq e^{\epsilon} \cdot \mathbb{P}[\mathcal{M}(\mathcal{U}') \in E] + \delta. \)
\end{definition}
In offline RL, each data point corresponds to a trajectory, and our goal is to design privacy-preserving algorithms that guarantee trajectory-level privacy.
We adopt $\rho$-zCDP (defined in Definition \ref{def: zCDP}) as our primary privacy notion due to its analytical simplicity, following \cite{qiao2024offline}. 
\begin{definition}[$\rho$-zCDP \cite{bun2016concentrated}]
\label{def: zCDP}
A randomized mechanism $\mathcal{M}$ satisfies $\rho$-Zero-Concentrated Differential Privacy ($\rho$-zCDP) if for all pairs of neighboring datasets $\mathcal{U}, \mathcal{U}'$ and all $\alpha > 1$:
\(R_\alpha\left(\mathcal{M}(\mathcal{U}) \,\|\, \mathcal{M}(\mathcal{U}')\right) \leq \rho \alpha, \)
where $R_\alpha$ is the Rényi divergence of order $\alpha$ of $\mathcal{M}(\mathcal{U})$ from $\mathcal{M}(\mathcal{U}^{\prime})$ \cite{van2014renyi}.
\end{definition}
Additional technical properties of $\rho$-zCDP (e.g.,
composition) are provided in Section 3 of the Appendix. 
We adopt $\rho$-zCDP in the subsequent analysis of privacy guarantees. It is worth noting that $\rho$-zCDP guarantees can be translated into $(\epsilon, \delta)$-DP through an appropriate parameter mapping. We refer interested readers to Lemma 2 in the Appendix for details. 

\section{Privacy Preserving RL with One-Sided Feedback}
\label{sec:algorithm}
In this section, we propose a differentially private algorithm named \texttt{POOL} for multi-dimensional RL with one-sided feedback that ensures $\rho$-zCDP.  The key algorithmic components are outlined in Section~\ref{sec: Key Ingredients}, while detailed analysis of the privacy mechanisms is deferred to Appendix due to space constraints.

\subsection{Design Overview}
\label{sec: Design overview}
Algorithm~\ref{alg1-offRL} provides an overview of \texttt{POOL}, which builds upon the framework introduced by~\cite{qiao2024offline}.  They constructs private value functions for RL in finite state-action spaces with full feedback. To extend this framework to multi-dimensional, continuous environments with one-sided feedback, we introduce a tailored partial discretization technique that addresses the complexity of continuous action spaces and the partial observability characteristic of one-sided feedback. In addition, we propose a multi-dimensional piecewise-linear approximation to represent the value function over discretized action zones, facilitating efficient and accurate estimation under privacy constraints.
In summary, our algorithm proceeds iteratively over each stage $h \in [H]$ with the following steps:

\begin{itemize}
    \item \textbf{Step 1: Discretization.} Apply the partial discretization technique to divide the multi-dimensional action space into $M$ zones based on the $\ell_2$ norm (Line~\ref{alg1:discretize} in Algorithm~\ref{alg1-offRL}).
    
    \item \textbf{Step 2: Private Estimation.} Construct differentially private value function estimates for each discretized region (Line~\ref{alg1:private-value} in Algorithm~\ref{alg1-offRL}).
    
    \item \textbf{Step 3: Piecewise-linear Approximation.} Interpolate the value function using a multi-dimensional piecewise-linear approximation based on sampled points within each zone (Line~\ref{alg1:piecewise-linear} in Algorithm~\ref{alg1-offRL}).
\end{itemize}
\subsection{Key Ingredients}
\label{sec: Key Ingredients}
In this section, we introduce three key ingredients of our proposed algorithm: the \emph{Gaussian mechanism}, a \emph{discretization strategy}, and a \emph{multi-dimensional piecewise-linear approximation}. These components are crucial for efficiently addressing the challenges posed by multi-dimensional RL problems with one-sided feedback.

\paragraph{\textbf{Gaussian Mechanism.}} 
The Gaussian mechanism is one of the most widely used techniques in the differential privacy literature for preserving privacy~\cite{dwork2014algorithmic}. A straightforward approach in multi-dimensional RL is to add Gaussian noise directly to sensitive data, but this often leads to severe performance degradation~\cite{chen2022privacy}. To address this, we adopt a more principled approach by applying the Gaussian mechanism to privatize estimates of visit counts, transition kernels, and value functions. Specifically, following \cite{qiao2024offline}, we add Gaussian noise to the counts of state-action pairs, with the variance of this noise calibrated according to the sensitivity and privacy parameters. To achieve $\rho$-zCDP, we build on the privatized counts from the  Gaussian mechanism to construct estimators, such as private transition probabilities denoted by $\widetilde{P}_h$ and private value functions denoted by $\widetilde{Q}_h$ and $\widetilde{V}_h$, respectively (for further details, see Section 2 in Appendix). This approach ensures differential privacy throughout the learning process while maintaining estimation accuracy.

\paragraph{\textbf{Discretization Strategy.}} 
A natural approach to solving continuous RL problems is to discretize the state-action space and apply tabular RL methods~\citep{qin2023sailing, gong2023bandits}. However, this approach scales poorly in multi-dimensional settings due to the exponential growth in discrete zones when each dimension is discretized independently.
To illustrate, consider a state-action space with $w + d$ dimensions, each discretized into $q$ bins. This results in $q^{w+d}$ discrete zones, which grows exponentially with the number of dimensions. For example, with $10$ dimensions and $5$ bins per dimension, the total reaches $5^{10}$, rendering the method impractical for multi-dimensional problems.

Another challenge in RL with one-sided feedback is the partial observability of state-action pairs and their corresponding transition kernels and value functions. Existing approaches often use partial discretization of the action space \cite{qin2023sailing, gong2023bandits}, where the observable region is divided into $M$ equal-length segments, and unobserved regions are assigned the estimated $Q$-values for the boundary. However, this method is tailored to one-dimensional settings and does not scale effectively to the multi-dimensional scenarios common in practice.

To address these challenges, we propose a scalable discretization strategy that partitions the $(w + d)$-dimensional state-action space into $M$ zones based on the $\ell_p$ norm, constrained to the observable region, i.e., $(s_h, a_h) < \boldsymbol{\lambda}_h$. For simplicity and geometric interpretability, we adopt the $\ell_2$ norm.A vector $b \in \mathbb{R}^{w + d}$ is assigned to zone $m \in [M]$ if and only if
$\frac{(w + d)(m-1)}{M} \leq \|b\|_2 < \frac{(w + d)m}{M},$
where $M$ is the discretization level, which influences the approximation error of the piecewise-linear model in bounding the performance loss.

This construction leads to intuitive geometric partitions, especially in low-dimensional cases.  For example, in one dimension, the space is divided into $M$ equal-length intervals ---an approach that aligns with standard discretization technique used in the literature \cite{qin2023sailing,gong2023bandits}.

This norm-based discretization preserves rotational symmetry and enables structured, zone-wise approximation. Compared to traditional axis-aligned discretization, it reduces the number of discrete regions, improving scalability for multi-dimensional RL tasks.

\paragraph{\textbf{Multi-dimensional Piecewise-Linear Approximation.}} 
Most existing research applies a piecewise-linear approximation to estimate value functions in one-dimensional RL. However, to the best of our knowledge, this approach has not been extended to multi-dimensional settings. To enable scalability in multi-dimensional setting, we iteratively construct a set of \(w + d\) linearly independent vectors \(\{b_{m}^1, \dots, b_{m}^{w + d}\}\) to span each discretized zone. Any finite-dimensional vector space admits a basis, and since each zone lies in the finite-dimensional space \( \mathbb{R}^{w+d} \), such a basis always exists (Theorem 1 in Appendix).  We begin by randomly selecting an initial vector \(b_{m}^1\) in zone \(m\). Subsequently, we iteratively construct additional vectors within zone \(m\), such as \(b_{m}^2\), ensuring each new vector is linearly independent of the previously selected ones. The process continues until a complete set of \(w + d\) linearly independent vectors \(\{b_{m}^1, \dots, b_{m}^{w + d}\}\) is obtained.

To improve computational efficiency, we employ orthogonal basis vectors, as their independence enables direct computation of projection coefficients without solving linear systems. This significantly accelerates interpolation. 
Once the \(w + d\) linearly independent vectors are identified, we apply the Gram-Schmidt orthogonalization process \cite{strang2022introduction} to construct orthogonal basis vectors.

For any state-action pair \((s_h, a_h)\) within a zone, we approximate its \(Q\)-value as a linear combination of the \(Q\)-values at the basis vectors. Specifically, we represent \((s_h, a_h)\) as 
\((s_h, a_h) = \sum_{j=1}^{w+d} m_j b_{m}^j.\) Then its \(Q\)-value is approximated by \( \underline{Q}_h(s_h, a_h) = \sum_{j=1}^{w+d} m_j \overline{Q}_h(b_{m}^j),\) where $\underline{Q}_h(s_h, a_h)$ denotes the approximated \(Q\)-value at any continuous state-action pair, and $\overline{Q} $ refers to the \(Q\)-value at discrete base points. The corresponding value function is then given by:
\(\widetilde{V}_h(s_h) = \max_{a_h} \underline{Q}_h(s_h, a_h).\) For state-action pairs that fall outside the observable boundary, i.e., \((s_h, a_h) \geq \boldsymbol{\lambda}_h\), the estimated \(Q\)-value is assigned a constant value equal to that at the boundary.

\begin{figure}[t]
\centering
\begin{algorithm}[H]
\caption{Privacy-Oriented One-Sided Learning (\texttt{POOL})}
\begin{algorithmic}[1]
\label{alg1-offRL}
\STATE \textbf{Input:} Offline dataset $\mathcal{D}$, reward $r$, constants $C_1=\sqrt{2}, C_2=16, C > 1$, failure probability $\delta$, zCDP budget $\rho$, $E_\rho$ and $\iota$
\STATE \textbf{Initialize:} Estimate $\widetilde{P}_h$, set $\widetilde{V}_{H+1} = 0$ 
\FOR{$h = H$ to $1$}
\IF{ $(s_h,a_h)<\boldsymbol{\lambda}_h$}
    \STATE Partition the state-action space into $K$ zones according to $\ell_2$ norm \label{alg1:discretize}
    \FOR{each zone $m \in [M]$ and basis vector $b_{m}^j$, $j \in [w + d]$}
        \STATE $\widetilde{Q}_h(b_{m}^j) \leftarrow r_h(b_{m}^j) + \widetilde{P}_h \cdot \widetilde{V}_{h+1}(b_{m}^j)$ \label{alg1:private-value}
        \STATE Compute $\Gamma_h(b_{m}^j)$ (or CH if $n \leq E_\rho$)
        \STATE $\widehat{Q}_h^p(b_{m}^j) \leftarrow \widetilde{Q}_h(b_{m}^j) - \Gamma_h(b_{m}^j)$ \label{pessimism}
        \STATE $\overline{Q}_h= \min \left\{\widehat{Q}_h^p, H-h+1\right\}^{+}$
    \ENDFOR
    \STATE Represent $(s_h, a_h)$ in zone $m$ as: $(s_h, a_h) = \sum_{j=1}^{w+d} m_j b_{m}^j$
    \STATE $\underline{Q}_h(s_h, a_h) = \sum_{j=1}^{w+d} m_j \overline{Q}_h(b_{m}^j)$ \label{alg1:piecewise-linear}

    \ENDIF
    \STATE $\underline{Q}_h(s_h, a_h) = \overline{Q}_h(\boldsymbol{\lambda}_h)$
  \STATE $\widetilde{\pi}_h(\cdot \mid s_h) = \arg\max_{\pi} \mathbb{E}_{a_h \sim \pi}[\underline{Q}_h(s_h, a_h)]$

\STATE $\widetilde{V}_h(s_h) = \mathbb{E}_{a_h \sim \widetilde{\pi}_h(\cdot \mid s_h)} \big[ \underline{Q}_h(s_h, a_h) \big]$
\ENDFOR
\STATE \textbf{Output:} Policy $\{\widetilde{\pi}_h\}_{h=1}^H$
\end{algorithmic}
\end{algorithm}
\end{figure}

\section{Theoretical Analysis}
\label{sec: theory}

We now present a formal analysis of Algorithm~\ref{alg1-offRL}, establishing its differential privacy guarantee and analyzing its sample complexity, i.e.,  the number of samples required to ensure both privacy preservation and near-optimal performance guarantee (Theorem~\ref{th2}).

We begin by introducing a mild regularity condition on the reward function.

\begin{assumption}
\label{asp: lipschitz} 
For all $h \in [H]$, the reward function $\widetilde{r}_h$ satisfies the following generalized Lipschitz condition:
\[
\left|\widetilde{r}_{h}\left( b_h\right)-\widetilde{r}_{h}\left(b_h^{\prime}\right)\right|  \leq l_h \left| \|b_h\|_2 - \| b_h^{\prime}\|_2 \right|,
\]
where $l_h$ denotes a generalized Lipschitz constant.
\end{assumption}
This condition is relatively mild, as it is satisfied by a broad class of distributions (see Appendix for details), and frequently arises in practical domains such as inventory control and capital management. For example, \cite{qin2023sailing} adopt a similar condition in a one-dimensional RL setting.  

Notably, this assumption modifies the traditional Lipschitz condition to depend on differences of norms rather than norms of differences. This adjustment reflects that the reward primarily depends on the magnitude of the state–action vector rather than its direction, yielding a less restrictive yet sufficient smoothness condition for multi-dimensional RL analysis. Furthermore, it facilitates bounded discretization error, enabling efficient learning in high-dimensional settings.

We first derive the following bound on the piecewise-linear approximation error:

\begin{theorem}[Bound on Piecewise-Linear Approximation Error]
\label{lemma:Q_difference_bound}
For any \( b_h = (s_h, a_h) \in [0, 1]^{w+d} \), the piecewise-linear approximation error is bounded as follows:
$\left| \underline{Q}_{h}(b_h) - \overline{Q}_{h}(b_h) \right|
\leq \frac{L_h \sqrt{w + d}}{M} + L_h \left| w + d - \|\boldsymbol{\lambda}_h\|_2 \right|,$ where $L_h = (H - h + 1)L$, and $L = \max_h l_h$.
\end{theorem}

Theorem~\ref{lemma:Q_difference_bound} highlights key factors influencing the accuracy of piecewise-linear approximation. Specifically,  the approximation error increases when feedback is highly censored (i.e., small \( \|\boldsymbol{\lambda}_h\|_2 \)) or when the state-action space is multi-dimensional, underscoring the challenges of one-sided feedback in multi-dimensional RL. Conversely, increasing the number of discretized zones $M$ can help reduce this error.

Next, we state our main theoretical result characterizing both the privacy and performance guarantees.

\begin{theorem}
\label{th2} (1) Algorithm~\ref{alg1-offRL} satisfies $\rho$-zCDP. (2) For any target accuracy $\alpha > 0$, let $M = \Theta\left(\frac{w + d}{\alpha}\right)$ denote the number of discretization zones. If the number of samples per stage satisfies
\(n = \tilde{O}\left( \frac{H^3 (1+E_\rho)}{\alpha} \right), \)
then the learned policy $\widetilde{\pi}$ satisfies
\( v^{\star} - \alpha - \tilde{O}\left(L_1 \left| w + d - \|\boldsymbol{\lambda}_h\|_2 \right|\right) \leq v^{\widetilde{\pi}} \leq v^{\star}\)
with probability at least $1 - \delta$, where $E_\rho = 4 \sqrt{\frac{H \log \left( \frac{4 H M^{2}(w + d)w}{\delta} \right)}{\rho}}$, $L_h = (H - h + 1)L$, and $L = \max_h l_h$.
\end{theorem}
Theorem~\ref{th2} provides an upper bound, across problem instances, on the minimal number of samples needed to ensure that the learned policy is both privacy-preserving and near-optimal. Such an upper bound is crucial for understanding the theoretical and practical efficiency of RL algorithms. It provides performance guarantees, ensuring that the learned policy is privacy-preserving and approximates the optimal policy within a desired accuracy level given sufficient data. It also offers insights into the sample efficiency of the method, particularly important in offline settings where data collection is constrained. Moreover, this upper bound can serve as a benchmark for comparing algorithms and guide the design of more efficient learning strategies by identifying critical factors—such as dimensionality and feedback sparsity, that impact learning complexity. According to the theorem, this bound scales polynomially with the episode length $H$, the discretization granularity $M$, the combined dimension $(w + d)$, and inversely with the privacy budget $\rho$. As expected, stronger privacy (i.e., smaller $\rho$) requires more samples to maintain the same level of optimality.

\paragraph{Quality of the Bound.}
The sample complexity bound of \texttt{POOL} is tight up to logarithmic factors in the full-feedback regime, which corresponds to the case where $\left| w + d - \|\boldsymbol{\lambda}_h\|_2 \right| = 0$. In this case, our upper bound $\tilde{O}\left( H^3 \alpha^{-2} \right)$ matches the information-theoretic lower bounds $\tilde{\Omega}\left( H^3 \alpha^{-2} \right)$ established by~\cite{sidford2018near}. Furthermore, in the special case where the problem is one-dimensional  ($d = 1$), our sample complexity bound aligns with the established result under one-sided feedback from~\cite{qin2023sailing}, validating our generalization to higher-dimensional settings as a significant advancement in the theory of RL under partial feedback. 
Finally, compared to the upper bound $\tilde{O}( (1+E_\rho) H^3 \alpha^{-2})$ established in~\cite{qiao2024offline} for private RL in one-dimensional tabular settings, our result achieves the same order of complexity, even in multi-dimensional MDPs with continuous state-action spaces. This highlights the dimensional scalability and robustness of our bound across diverse and more complex RL environments.

\paragraph{Remark.} 
The sample complexity bound  for the one-sided feedback case remains valid in the full-feedback setting, corresponding to $\left| w + d - \|\boldsymbol{\lambda}_h\|_2 \right|=0$ for all $h \in[H]$. 
This demonstrates the robustness of \texttt{POOL} across both feedback regimes, highlighting its versatility and suitability for large-scale, continuous RL tasks with diverse feedback structures.



\section{Experiments}
\label{sec:experiments}

We conduct a comprehensive evaluation of \texttt{POOL} to address the following research questions:

\begin{itemize}
    \item \textbf{RQ1:} How does \texttt{POOL} perform compared to baseline methods?  
    \item \textbf{RQ2:} How do key hyperparameters of \texttt{POOL} affect its performance?  
    \item \textbf{RQ3:} How effective and efficient is the discretization strategy employed in \texttt{POOL}?  
\end{itemize}

\subsection{Experimental Setup}

We evaluate \texttt{POOL} across three aspects: comparison with baseline methods, sensitivity to key hyperparameters, and discretization efficiency. Experiments are conducted on lost-sales inventory control problems using both synthetic data and real-world data from the Rossmann Sales dataset~\citep{qin2023sailing,gong2023bandits}. Detailed descriptions of the problem settings and synthetic data generation are provided in the Appendix.

\subsubsection{Evaluation Metrics}
\label{evaluation_metrics}
Performance is quantified using the \emph{Relative Optimality Gap}. Let $c_i$ denote the cumulative expected cost of policy $i$, and $c_i^*$ denote the benchmark cost obtained via sample average approximation with $n=10{,}000$ samples and zero initial inventory. The relative gap is defined as:
\[
\text{Relative Gap} = \frac{c_i - c_i^*}{c_i^*} \times 100\%.
\]
All results report the mean and standard deviation over 10 independent runs.

\subsubsection{Baseline Methods}
We compare \texttt{POOL} with the following baselines:

\begin{enumerate}
    \item \textbf{NonPrivate:} A non-private algorithm without added noise.  
    \item \textbf{Input Perturbation (IP):} Adds Gaussian noise to the input data.  
    \item \textbf{Output Perturbation (OP):} Adds Gaussian noise to the output data.  
\end{enumerate}

\subsection{Performance Comparison (RQ1)}

We evaluate \texttt{POOL} against IP and OP under varying privacy budgets $\rho \in \{0.1, 1, 5, 10, 20, 40\}$. Each experiment is repeated 10 times. Figure~\ref{fig:pool_privacy_performance} reports the relative optimality gaps for both synthetic and real-world datasets. Across all settings, \texttt{POOL} consistently outperforms standard private baselines, which suffer from substantial performance degradation, and closely approaches the performance of the non-private algorithm. This demonstrates that our private method is significantly more effective than conventional input/output perturbation techniques while still preserving privacy.

\begin{figure}[t]
    \centering
    \includegraphics[width=0.48\textwidth]{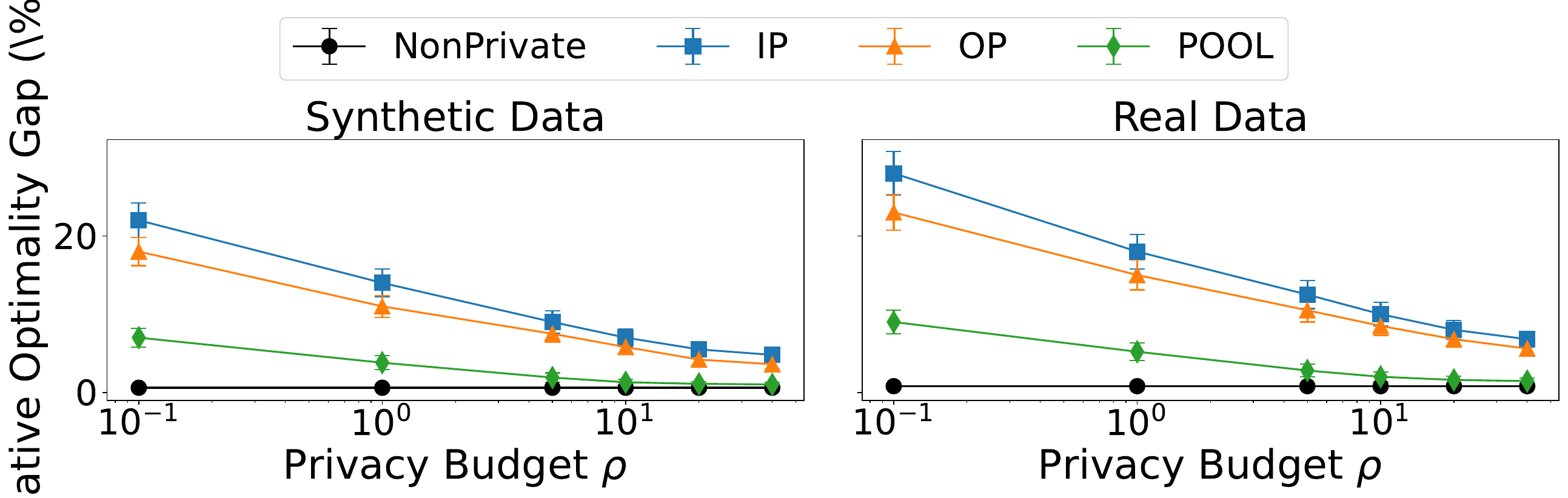}
    \caption{Relative optimality gap of \texttt{POOL} and baseline methods under varying privacy budgets $\rho$. \textbf{Left:} synthetic data; \textbf{Right:} real-world data. The x-axis is logarithmically scaled.}
    \label{fig:pool_privacy_performance}
\end{figure}

\subsection{Hyperparameter Analysis (RQ2)}

We study the effect of key hyperparameters on \texttt{POOL} using synthetic data: horizon length $H \in \{5,10,15,20,40\}$, feedback parameter $|\lambda| \in \{0.5,0.6,0.7,0.8,0.9,1.0\}$, state-action dimensionality $w+d \in \{2,4,6,8,16\}$, and discretization granularity $M \in \{50,100,200,400,800\}$. Each configuration is evaluated over multiple independent runs, reporting mean relative gaps and standard deviations.

\textbf{Horizon length ($H$).} Longer horizons generally lead to higher relative gaps due to increased sample complexity (Figure~\ref{fig:rl_experiment_factors}, top-left).  

\textbf{Feedback parameter ($|\lambda|$).} One-sided feedback increases learning difficulty, whereas larger $|\lambda|$ improves performance (top-right).  

\textbf{State-action dimensionality ($w+d$).} Higher-dimensional state-action spaces result in larger gaps, reflecting the increased complexity of multi-dimensional reinforcement learning problems (bottom-left).  

\textbf{Discretization granularity ($M$).} Finer discretization consistently reduces the relative gap, mitigating the curse of dimensionality and improving learning efficiency (bottom-right).

\begin{figure}[htbp]
    \centering
    \includegraphics[width=0.5\textwidth]{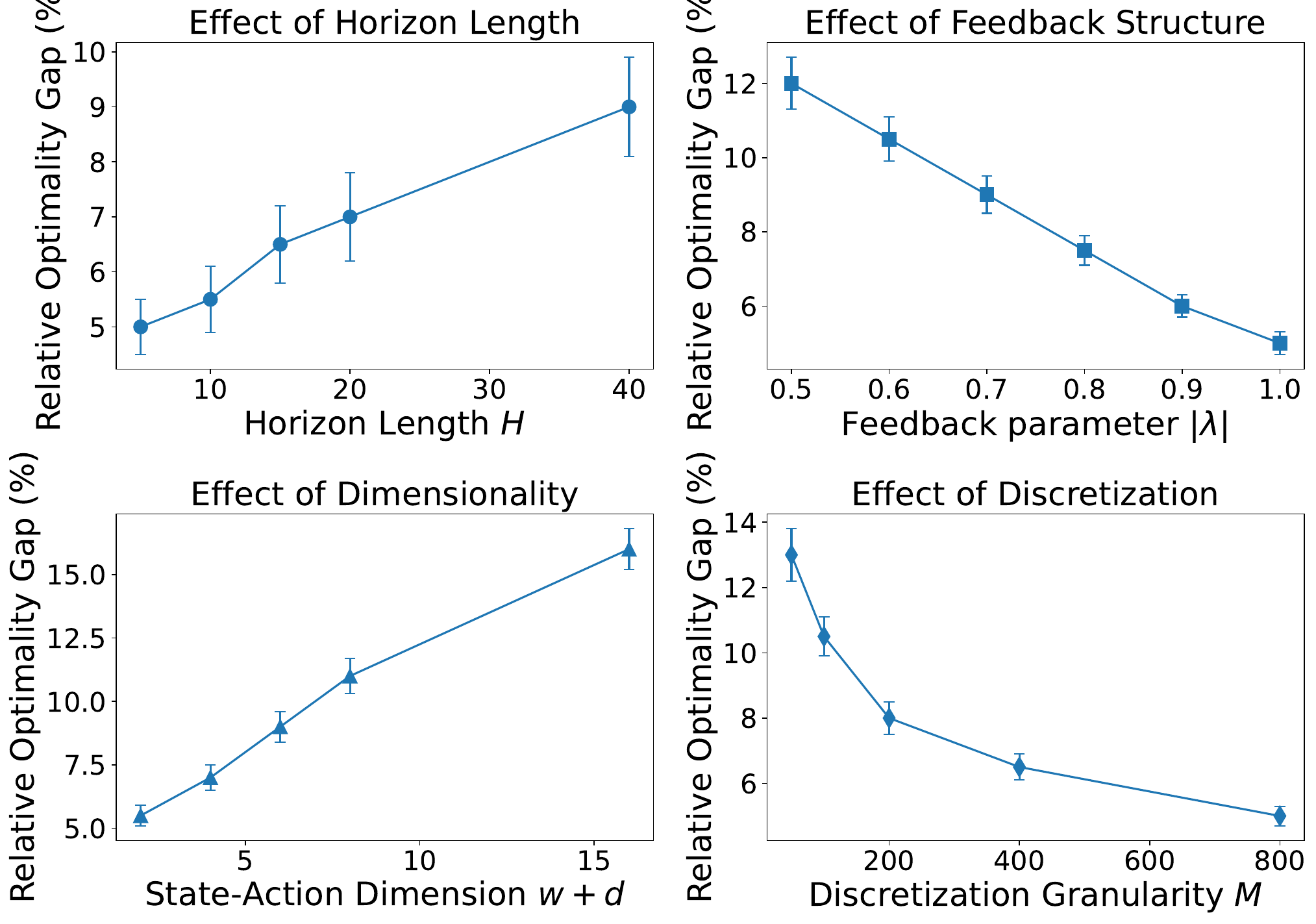}
    \caption{Impact of key hyperparameters on relative optimality gap. \textbf{Top-left:} Horizon length $H$. \textbf{Top-right:} Feedback parameter $|\lambda|$. \textbf{Bottom-left:} State-action dimensionality $w+d$. \textbf{Bottom-right:} Discretization granularity $M$. Error bars indicate standard deviations.}
    \label{fig:rl_experiment_factors}
\end{figure}

\subsection{Effectiveness and Efficiency of Discretization (RQ3)}

We compare \texttt{POOL}'s discretization strategy against standard grid-based methods in terms of relative optimality gap and computational time. Experiments are conducted on synthetic and real-world datasets using discretization granularities $M \in \{50,100,200,400,800\}$, and running times are measured in minutes. Each experiment is repeated multiple times to account for variance.

\textbf{Relative Optimality Gap:} \texttt{POOL} achieves consistently lower gaps than grid-based methods for all $M$, indicating it effectively captures the critical structure of the state-action space and maintains near-optimal performance even with coarser discretization.

\textbf{Computational Efficiency:} \texttt{POOL} substantially reduces computational cost. While grid-based approaches exhibit superlinear increases in running time with $M$, \texttt{POOL} scales more gracefully, making it suitable for large-scale inventory control problems.

Overall, \texttt{POOL} consistently outperforms grid-based discretization in both effectiveness and efficiency (Figure~\ref{fig:pool_vs_grid}).

\begin{figure}[t]
    \centering
    \includegraphics[width=0.5\textwidth]{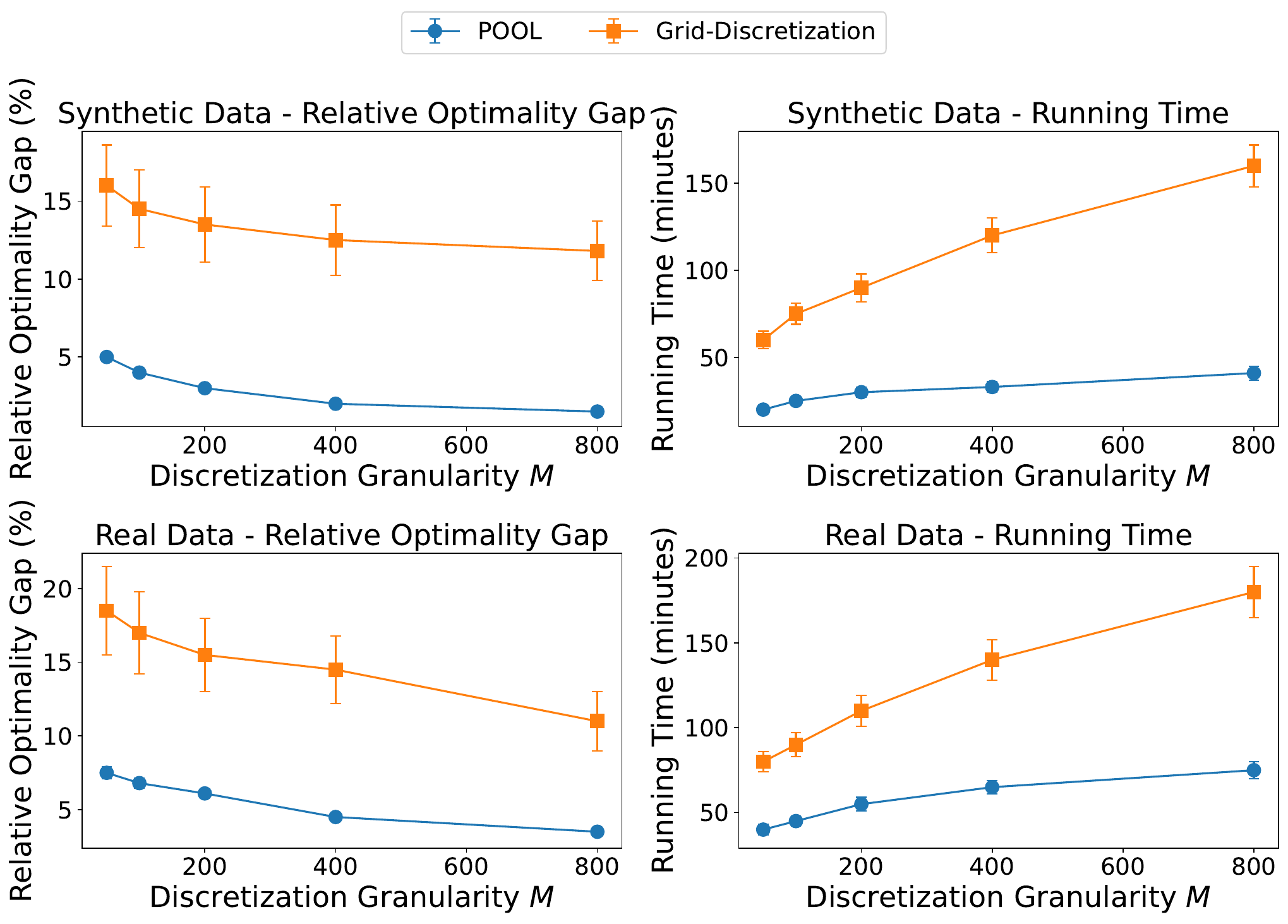}
    \caption{Comparison of \texttt{POOL} and grid-based discretization. \textbf{Left:} relative optimality gap; \textbf{Right:} running time. \texttt{POOL} consistently achieves lower gaps and faster computation across datasets.}
    \label{fig:pool_vs_grid}
\end{figure}

\section{Conclusion}
\label{sec:conclusion}

This paper initiates the study of privacy-preserving RL in multi-dimensional decision making with one-sided feedback---an important yet underexplored problem with practical relevance in areas like business optimization and personalized recommendations. We propose \texttt{POOL}, a novel algorithm that ensures differential privacy while achieving strong sample efficiency. Our theoretical analysis and empirical results show that \texttt{POOL} effectively balances privacy and learning performance, offering a solid foundation for safe and scalable decision-making in sensitive environments. Future work includes extending the framework to online learning with evolving, partial feedback and developing contextual privacy-preserving RL methods that adapt to dynamic, feature-rich environments.

\section{Acknowledgement}
This work was supported by the DTC Singapore Innovation Grant (DTC-IGC-07) under the project Multi-Party Evaluation: Blockchain-Enabled AI Safety and Verification Framework.

\appendix

\section{Additional Experiment Details}
We apply our method to an inventory management problem, where any demand exceeding the available inventory is lost. This setting can be naturally formulated as an MDP with one-sided feedback. Details of the inventory control application are provided in the Appendix (Example Application: Inventory Control). 

\paragraph{Experimental Setting.} 
We evaluate \texttt{POOL} under varying state-action space dimensions $w+d \in \{2,4,6,8,16\}$ and horizon lengths $H \in \{5,10,15,20,40\}$. 
Discretization granularity is varied as $M \in \{50,100,200,400,800\}$. 
The demand upper bound is fixed at $\overline{D}=100$. 
Stage-wise holding and backorder costs are independently drawn from uniform distributions, $h_h \sim U[0,0.5]$ and $b_h \sim U[0,0.5]$, 
and the stage-wise demand follows $D_h \sim U[0,\overline{D}]$. 
For privacy-preserving experiments, we evaluate budgets $\rho \in \{0.1,1,5,10,20,40\}$ and compare the results with the non-private baseline.

\paragraph{Experiment Environment.} 
All experiments are implemented in Python and conducted on an HP EliteBook 830 G6 with 8 GB RAM and an Intel(R) Core(TM) i7-8565U CPU.

\paragraph{Experimental Setting on Real Data.}
We  evaluate the performance of the proposed  algorithms on the real-world datasets.
The dataset is split into two equal parts: the first half is used as the training set to learn the inventory policy, while the second half is used as the test set to evaluate performance. The cost parameters are generated following the same procedure as in the synthetic experiments.
To assess the effectiveness of the learned policies, we use the following definition of the optimality gap:
\[
\text{Relative Optimality Gap}^\pi = V_{\text{test}}^{\text{saa}} - V_{\text{test}}^\pi,
\]
where \( V_{\text{test}}^{\text{saa}} \) denotes the total cost on the test set incurred by the Sample Average Approximation (saa) method, and \( V_{\text{test}}^\pi \) represents the total cost incurred by the policy \( \pi \) learned by a differentially private algorithm.

We evaluate POOL algorithm using the Rossmann Store Sales dataset. This publicly available dataset from Kaggle contains daily sales records from 1,115 stores, spanning the period from January 1, 2013, to July 31, 2015. For our analysis, we focus on Store 1, using its daily sales as the input for the inventory control problem.

We analyze the temporal patterns of the sales data using the autocorrelation function (ACF). The ACF results exhibit a clear weekly cycle, indicating strong periodicity in the demand pattern. This observation aligns with prior findings in \citep{gong2023bandits,qin2023sailing}, which also report weekly seasonality in retail sales data. Based on this cyclical behavior, we assume a replenishment decision is made weekly, and thus model the inventory problem using a finite horizon of \( H = 7 \) periods leading up to each Monday.

\section{Example Applications: Inventory Control}
\label{appendix:inventory}

Inventory control is a fundamental problem in operations management \citep{qin2023sailing,gong2020provably,qin2022data}. 
In this setting, the action corresponds to the multi-product inventory ordering decision. At each time step $h$, the retailer observes the current inventory level $x_h \in \mathbb{R}^d$ and chooses an order-up-to level $y_h \in \mathbb{R}^d$ such that $y_h \geq x_h$. Without loss of generality, we assume zero ordering cost. The ordered quantity $y_h - x_h$ is replenished instantaneously. Then, an independent random demand $D_h \in \mathbb{R}^d$ is realized, with each component bounded in $[0, \overline{D}]$. 

We study two canonical models: the \textbf{backlog model} and the \textbf{lost-sales model}. In both cases, unfulfilled demand incurs a per-unit backlogging cost $b_h\in \mathbb{R}^d$, while excess inventory incurs a per-unit holding cost $h_h\in \mathbb{R}^d$.

\paragraph{Backlogged Model.}  
If demand $D_h$ exceeds the inventory level $y_h$, the excess demand is backlogged. Consequently, the inventory at the start of the next period becomes $x_{h+1} = y_h - D_h < 0$. The period-$h$ reward is defined as the negative of the total cost:
\[
r_h(y_h, D_h) = -\left( h_h^\top \cdot (y_h - D_h)^+ + b_h^\top \cdot (D_h - y_h)^+ \right),
\]
where $(\cdot)^+$ denotes the positive part.  
This backlogged model offers \textbf{full feedback}, as both the realized demand and inventory outcomes are observable after each period.

\paragraph{Lost-Sales Model.}  
When demand exceeds inventory in the lost-sales model, the unfulfilled demand is lost and unobserved rather than backlogged. Thus, the actual reward is not fully observable due to the unknown surplus demand term $(D_h - y_h)^+$. To address this, we follow \citet{qin2023sailing,gong2020provably} and adopt a \textbf{virtual reward}:
\[
r_h(y_h, D_h) = -\left( h_h^\top \cdot (y_h - D_h)^+ - b_h^\top \cdot \min(y_h, D_h) \right),
\]
which can be computed using observable quantities. This setting is characterized by \textbf{one-sided feedback}, as only partial information about the demand is available.

\paragraph{Results.}  
In both the backlogged and lost-sales models, the \texttt{POOL} algorithm achieves a sample complexity of 
\(\tilde{O}\left((1+ E_{\rho} ) H^3 \alpha^{-2} \right).\)

\section{Typical examples for Assumption 1}
Typical examples  that satisfy Assumption 1 include:
\begin{itemize}
    \item Constant functions: $\widetilde r_h(b_h) = C$, with $l_h=0$.
    \item Affine functions of the norm: $\widetilde r_h(b_h) = a \|b_h\|_2 + c$, with $l_h = |a|$.
    \item Saturated or truncated linear functions: $\widetilde r_h(b_h) = \min\{A, a \|b_h\|_2 + c\}$, with $l_h = |a|$.
    \item Smooth nonlinear functions composed with the norm, e.g., $\widetilde r_h(b_h) = \sin(\|b_h\|_2)$ or $\widetilde r_h(b_h) = \tanh(\|b_h\|_2)$, with $l_h$ equal to the maximal derivative of the scalar function.
    \item Any univariate Lipschitz function $g_h:[0,\infty)\to \mathbb R$ composed with the norm: $\widetilde r_h(b_h) = g_h(\|b_h\|_2)$.
\end{itemize}
\section{Private Components}
\label{sec: private components}
We outline the construction of private components within our differentially private RL framework. These include \emph{private counts}, \emph{private transition kernels}, and \emph{private value functions}, all of which are essential to preserving privacy while maintaining algorithmic performance.

\paragraph{\textbf{Private Counts.}}  
To construct private counts over state-action pairs, we employ the Gaussian mechanism, defined as follows:

\begin{Definition}[Gaussian Mechanism~\citep{dwork2014algorithmic}]
\label{def Gaussian Mechanism}
The Gaussian mechanism $\mathcal{M}$ with noise level $\sigma$ is defined as:
\[
\mathcal{M}(U) = f(U) + \mathcal{N}(0, \sigma^2 I_d),
\]
where $\mathcal{N}(0, \sigma^2 I_d)$ denotes a $d$-dimensional Gaussian noise with zero mean and covariance matrix $\sigma^2 I_d$.
\end{Definition}
This noise is calibrated based on the sensitivity of the function $f$ and the privacy parameter $\rho$, ensuring $\rho$-zCDP. 

Given dataset 
\(\mathcal{D} = \left\{ (s_h^\tau, a_h^\tau, r_h^\tau, s_{h+1}^\tau) \right\}_{\tau \in [n]}^{h \in [H]},\)
we define the visit counts as:
\[
n_{s_h, a_h} := \sum_{\tau=1}^n \mathbf{1}\left[(s_h^\tau, a_h^\tau) = (s_h, a_h)\right],
\]
\[
n_{s_h, a_h, s_{h+1}} := \sum_{\tau=1}^n \mathbf{1}\left[(s_h^\tau, a_h^\tau, s_{h+1}^\tau) = (s_h, a_h, s_{h+1})\right].
\]
To preserve privacy, we add independent Gaussian noise to these counts:
\begin{equation}
\label{n}
n_{s_h, a_h}' = \left\{ n_{s_h, a_h} + \mathcal{N}(0, \sigma^2) \right\}^{+},
\end{equation}
\[
n_{s_h, a_h, s_{h+1}}' = \left\{ n_{s_h, a_h, s_{h+1}} + \mathcal{N}(0, \sigma^2) \right\}^{+}\]
where $\sigma^2 = \frac{2H}{\rho}$, and $\{a\}^+ = \max(a, 0)$ ensures non-negativity.

To enforce consistency between marginal and joint counts, we solve the following optimization problem~\citep{qiao2024offline}:
\[
\left\{ \widetilde{n}_{s_h, a_h, s'} \right\}_{s' \in \mathcal{S}} = 
\arg\min_{\left\{x_{s'} \right\}_{s' \in \mathcal{S}}} 
\max_{s' \in \mathcal{S}} \left| x_{s'} - n_{s_h, a_h, s'}' \right|,
\]
subject to:
\[
\left| \sum_{s' \in \mathcal{S}} x_{s'} - n_{s_h, a_h}' \right| \leq \frac{E_\rho}{2}, 
\quad \text{and} \quad x_{s'} \geq 0 \quad \forall s' \in \mathcal{S},
\]
where 
\(E_\rho = 4 \sqrt{\frac{H \log \frac{4 H M^2w(w+d)}{\delta}}{\rho}}\)
is a high-probability uniform bound on the noise, and $K$ is a discretization parameter.

Finally, the private count for $(s_h, a_h)$ is defined as:
\begin{equation}
\label{N}
\widetilde{n}_{s_h, a_h} := \sum_{s' \in \mathcal{S}} \widetilde{n}_{s_h, a_h, s'}.
\end{equation}
\paragraph{\textbf{Private Transition Kernels.}}  
Using the private counts from~\eqref{N}, we define the private transition kernel $\widetilde{P}_h$ as:
\begin{equation}
\label{private transition1}
\widetilde{P}_h(s' \mid s_h, a_h) =
\begin{cases}
\frac{\widetilde{n}_{s_h, a_h, s'}}{\widetilde{n}_{s_h, a_h}}, & \text{if } \widetilde{n}_{s_h, a_h} > E_\rho, \\
\frac{1}{Mw}, & \text{otherwise}.
\end{cases}
\end{equation}
This ensures that the transition kernel remains a valid probability distribution, which is necessary for subsequent variance-based analysis, such as the Bernstein-type pessimism term.


\paragraph{\textbf{Private Value Functions.}}  
We denote by $\widetilde{\pi}$ the policy derived from our private estimators, and define the associated value functions $\widetilde{V}_h(\cdot)$ and $\widetilde{Q}_h(\cdot, \cdot)$.

To ensure privacy, we apply the Gaussian mechanism to the value estimation process, constructing $\widetilde{V}_h$ and $\widetilde{Q}_h$ using the private transition kernels from~\eqref{private transition1} and the noisy counts from~\eqref{n} and~\eqref{N}. These private value functions serve as core components in our policy construction.

By incorporating these private components, we demonstrate in Theorem 2 that our algorithm satisfies differential privacy, effectively balancing the trade-off between privacy and learning performance.

\section{Supporting Definitions and Lemmas for Differential Privacy}
\label{appendix: DP-lemmas}

A fundamental concept in the design of differentially private algorithms is the \emph{sensitivity} of a function, which quantifies the maximum change in the function's output when applied to neighboring datasets.

\begin{Definition}[$\ell_2$ Sensitivity \citep{dwork2014algorithmic}]
The $\ell_2$ sensitivity of a function $f: \mathbb{N}^{\mathcal{X}} \to \mathbb{R}^d$ is defined as
\begin{equation}
   \Delta_2(f) = \sup_{\text{neighboring } U, U'} \left\| f(U) - f(U') \right\|_2 \,.
\end{equation}
\end{Definition}

We next present the Gaussian mechanism, which perturbs the output of $f$ by adding Gaussian noise calibrated to its $\ell_2$ sensitivity.

\begin{Lemma}[Gaussian Mechanism \citep{dwork2014algorithmic}]
\label{lemma: rho-zcdp}
Let $f: \mathbb{N}^{\mathcal{X}} \to \mathbb{R}^d$ be a function with $\ell_2$ sensitivity $\Delta_2$. Then, for any $\rho > 0$, the Gaussian mechanism with variance $\sigma^2 = \frac{\Delta_2^2}{2\rho}$ satisfies $\rho$-zero-concentrated differential privacy ($\rho$-zCDP). Moreover, for any $0 < \epsilon, \delta < 1$, setting $\sigma = \frac{\Delta_2}{\epsilon} \sqrt{2 \log \frac{1.25}{\delta}}$ ensures $(\epsilon, \delta)$-differential privacy.
\end{Lemma}

\begin{Lemma}[Relationship Between DP and $\rho$-zCDP \citep{bun2016concentrated}]
\label{lemma: relationship}
Let $M$ be a randomized mechanism.
\begin{itemize}
    \item If $M$ satisfies $(\epsilon, 0)$-differential privacy, then it satisfies $\rho$-zCDP with $\rho = \frac{1}{2} \epsilon^2$.
    \item If $M$ satisfies $\rho$-zCDP, then for any $\delta > 0$, it satisfies $(\rho + 2\sqrt{\rho \log(1/\delta)}, \delta)$-differential privacy.
\end{itemize}
\end{Lemma}

\begin{Lemma}[Post-Processing for $\rho$-zCDP \citep{bun2016concentrated}]
\label{lemma: post-processing}
Let $\mathcal{A}: \mathcal{X}^n \rightarrow \mathcal{Y}$ be a randomized mechanism satisfying $\rho$-zCDP, and let $\mathcal{A}': \mathcal{X}^n \times \mathcal{Y} \rightarrow \mathcal{Z}$ be another mechanism satisfying $\rho'$-zCDP. Define the composed mechanism $\mathcal{A}'': \mathcal{X}^n \rightarrow \mathcal{Z}$ as 
\[
\mathcal{A}''(x) = \mathcal{A}'(x, \mathcal{A}(x)).
\]
Then $\mathcal{A}''$ satisfies $(\rho + \rho')$-zCDP.
\end{Lemma}

\begin{Lemma}[Composition of $\rho$-zCDP \citep{bun2016concentrated}]
\label{lemma: composition}
Let $\mathcal{A}: \mathcal{U}^n \rightarrow \mathcal{Y}$ and $\mathcal{A}': \mathcal{U}^n \rightarrow \mathcal{Z}$ be randomized mechanisms satisfying $\rho$-zCDP and $\rho'$-zCDP, respectively. Define the composed mechanism $\mathcal{A}'': \mathcal{U}^n \rightarrow \mathcal{Y} \times \mathcal{Z}$ by
\[
\mathcal{A}''(U) = \left( \mathcal{A}(U), \mathcal{A}'(U) \right).
\]
Then $\mathcal{A}''$ satisfies $(\rho + \rho')$-zCDP.
\end{Lemma}

\section{Supporting Lemmas for Linear Algebra}

\begin{Theorem}[\cite{strang2022introduction}]
\label{thm: basis vector}
    Every finite-dimensional vector space has a basis, a set of vectors that are linearly independent and span the entire vector space.
\end{Theorem}

\section{Proof of Theoretical Results}

\subsection{Proof of the privacy guarantee}
\label{sec: privacy-proof}
The privacy guarantee of  Algorithm 1 is summarized by Lemma \ref{lemma: offRL} below.
\begin{Lemma}[Privacy Analysis of Algorithm 1]
\label{lemma: offRL}
    Algorithm 1 satisfies $\rho$-zCDP.
\end{Lemma}
\begin{proof}
The $\ell_2$ sensitivity of $\left\{n_{s_h, a_h}\right\}$ is $\sqrt{2 H}$. According to Lemma \ref{lemma: rho-zcdp}, the Gaussian Mechanism used on $\left\{n_{s_h, a_h}\right\}$ with $\sigma^2=\frac{2 H}{\rho}$ satisfies $\frac{\rho}{2}$-zCDP. Similarly, the Gaussian Mechanism used on $\left\{n_{s_h, a_h, s_{h+1}}\right\}$ with $\sigma^2=\frac{2 H}{\rho}$ also satisfies $\frac{\rho}{2}$-zCDP. Combining these two results, due to the composition of $\rho$-zCDP (Lemma \ref{lemma: composition}), the construction of $\left\{n^{\prime}\right\}$ satisfies $\rho$-zCDP. Finally,  Algorithm 1 satisfies $\rho$-zCDP because the output $\widetilde{\pi}$ is post processing of $\left\{n^{\prime}\right\}$.   

This completes the proof.
\end{proof}

\subsection{Proof of Bound of Piecewise-Linear Approximation Error}

\begin{Lemma}
\label{WUV censored3}
Based on Assumption 1, we have that for all $b_h=(s_h, a_h) \in[0, 1]^{w+d}$ and $ b_h^{\prime} =(s_h^\prime, a_h^{\prime}) \in[0, 1]^{w+d}$,
\begin{align}
\left|\widetilde{Q}_{h}\left( b_h\right)-\widetilde{Q}_{h}\left(b_h^{\prime}\right)\right|
&=\left|\widetilde{Q}_{h}\left( s_h, a_h\right)-\widetilde{Q}_{h}\left(s_h^{\prime}, a_h^{\prime}\right)\right| 
\end{align}
$$\leq L_h\left| \left\| b_h \right\|_2 - \left\| b_h^{\prime} \right\|_2 \right|,$$
where $L_h=(H-h+1)L$ and $L=\max l_h$.
\end{Lemma}

\begin{proof}
We use mathematical induction to show it.

When $k=H$, based on Assumption 1, we have $\widetilde{Q}_H(s_H, a_H)=r_H(s_H, a_H)$ satisfies
\[
\left|r_{H}(b_H) - r_{H}(b_H')\right| 
= \left|r_{H}(s_H, a_H) - r_{H}(s_H', a_H')\right| 
\]
$$\leq l_H \left| \left\| b_H \right\|_2 - \left\| b_H' \right\|_2 \right|,$$
so $L_H = l_H$.

Suppose for $k = h+1$, the bound holds:
\[
\left|\widetilde{Q}_{h+1}(b_{h+1}) - \widetilde{Q}_{h+1}(b_{h+1}')\right|
\leq L_{h+1} \left| \left\| b_{h+1} \right\|_2 - \left\| b_{h+1}' \right\|_2 \right|.
\]

We now show the result for $k=h$:
\begin{align*}
&\left|\widetilde{Q}_h(b_h) - \widetilde{Q}_h(b_h')\right|
= \left|\widetilde{Q}_h(s_h, a_h) - \widetilde{Q}_h(s_h', a_h')\right| \\
&= \left| \widetilde{r}_h(s_h, a_h) + \widetilde{P}_h \widetilde{Q}_{h+1}(s_h, a_h)
- \widetilde{r}_h(s_h', a_h') - \widetilde{P}_h \widetilde{Q}_{h+1}(s_h', a_h') \right| \\
&\leq \left| \widetilde{r}_h(s_h, a_h) - \widetilde{r}_h(s_h', a_h') \right|
+ \left| \widetilde{P}_h \widetilde{Q}_{h+1}(s_h, a_h) - \widetilde{P}_h \widetilde{Q}_{h+1}(s_h', a_h') \right| \\
&\leq L \left| \left\| b_h \right\|_2 - \left\| b_h' \right\|_2 \right|
+ (H - h) L \left| \left\| b_h \right\|_2 - \left\| b_h' \right\|_2 \right| \\
&= (H - h + 1) L \left| \left\| b_h \right\|_2 - \left\| b_h' \right\|_2 \right| = L_h \left| \left\| b_h \right\|_2 - \left\| b_h' \right\|_2 \right|.
\end{align*}
Thus, $\widetilde{Q}_h(s_h, a_h)$ satisfies $$\left|\widetilde{Q}_{h}\left( b_h\right)-\widetilde{Q}_{h}\left(b_h^{\prime}\right)\right|
=\left|\widetilde{Q}_{h}\left( s_h, a_h\right)-\widetilde{Q}_{h}\left(s_h^{\prime}, a_h^{\prime}\right)\right| 
$$
$$\leq L_h\left| \left\| b_h \right\|_2 - \left\| b_h^{\prime} \right\|_2 \right|.$$

This completes the proof.
\end{proof}

\begin{Lemma}
\label{WUV censored4}
Based on Assumption 1, we have that for all $b_h=(s_h, a_h) \in[0, 1]^{w+d}$ and $ b_h^{\prime} =(s_h^\prime, a_h^{\prime}) \in[0, 1]^{w+d}$,
\begin{align}
\left|\overline{Q}_{h}\left( b_h\right)-\overline{Q}_{h}\left(b_h^{\prime}\right)\right|
&=\left|\overline{Q}_{h}\left( s_h, a_h\right)-\overline{Q}_{h}\left(s_h^{\prime}, a_h^{\prime}\right)\right|\\
&
\leq L_h\left| \left\| b_h \right\|_2 - \left\| b_h^{\prime} \right\|_2 \right|,
\end{align}
where $L_h=(H-h+1)L$ and $L=\max l_h$.
\end{Lemma}

\begin{proof}
By Lemma \ref{WUV censored3}, we have
\[
\left|\widetilde{Q}_{h}\left( b_h\right)-\widetilde{Q}_{h}\left(b_h^{\prime}\right)\right| 
\leq L_h\left| \left\| b_h \right\|_2 - \left\| b_h^{\prime} \right\|_2 \right|.
\]

Recall that $\overline{Q}_h$ is defined as 
 $\overline{Q}_h= \min \left\{\widetilde{Q}_h(b_{m}^j) - \Gamma_h(b_{m}^j), H-h+1\right\}^{+}$.
 
Since minimization and truncation preserves the generailized Lipschitz property, we directly have
\[
\left|\overline{Q}_{h}\left( b_h\right)-\overline{Q}_{h}\left(b_h^{\prime}\right)\right|
\leq L_h\left| \left\| b_h \right\|_2 - \left\| b_h^{\prime} \right\|_2 \right|.
\]

This completes the proof.
\end{proof}

\begin{Lemma}
\label{WUV censored5}
Based on Assumption 1, we have that for all $b_h=(s_h, a_h) \in[0, 1]^{w+d}$ and $ b_h^{\prime} =(s_h^\prime, a_h^{\prime}) \in[0, 1]^{w+d}$,
\begin{align}
\left|\underline{Q}_{h}\left( b_h\right)-\underline{Q}_{h}\left(b_h^{\prime}\right)\right|
&=\left|\underline{Q}_{h}\left( s_h, a_h\right)-\underline{Q}_{h}\left(s_h^{\prime}, a_h^{\prime}\right)\right| \\
&
\leq L_h\left| \left\| b_h \right\|_2 - \left\| b_h^{\prime} \right\|_2 \right|,
\end{align}
where $L_h=(H-h+1)L$ and $L=\max l_h$.
\end{Lemma}

\begin{proof}
By Lemma \ref{WUV censored4}, we know that $\overline{Q}_h$ satisfies:
\[
\left|\overline{Q}_{h}\left( b_h\right)-\overline{Q}_{h}\left(b_h^{\prime}\right)\right|
\leq L_h\left| \left\| b_h \right\|_2 - \left\| b_h^{\prime} \right\|_2 \right|.
\]

Recall that $\underline{Q}_{h}$ is defined as a censored or piecewise-linear version of $\overline{Q}_h$:
\[
\underline{Q}_h(s_h, a_h) =
\begin{cases}
\displaystyle \sum_{j=1}^{w+d} m_j \, \overline{Q}_h(b_m^j), & \text{if } (s_h, a_h) < \boldsymbol{\lambda}_h, \\[2mm]
\overline{Q}_h(\boldsymbol{\lambda}_h), & \text{if } (s_h, a_h) \ge \boldsymbol{\lambda}_h.
\end{cases}
\]

\noindent
Since $\underline{Q}_h$ is either a linear combination of generalized $L_h$-Lipschitz functions or a constant (in the truncated region), it preserves the generalized Lipschitz property. Therefore, for all $b_h, b_h' \in [0,1]^{w+d}$, we have
\[
\left|\underline{Q}_{h}\left( b_h\right)-\underline{Q}_{h}\left(b_h^{\prime}\right)\right|
\leq L_h\left| \left\| b_h \right\|_2 - \left\| b_h^{\prime} \right\|_2 \right|.
\]

This completes the proof.
\end{proof}

\begin{Lemma}
\label{WUV censored6}
Based on Assumption 1, we have that for all $b_h=(s_h, a_h) \in[0, 1]^{w+d}$ and $ b_h^{\prime} =(s_h^\prime, a_h^{\prime}) \in[0, 1]^{w+d}$,
\begin{align}
\left|\widetilde{V}_{h}\left( b_h\right)-\widetilde{V}_{h}\left(b_h^{\prime}\right)\right|
&=\left|\widetilde{V}_{h}\left( s_h, a_h\right)-\widetilde{V}_{h}\left(s_h^{\prime}, a_h^{\prime}\right)\right| 
\\& \leq L_h\left| \left\| b_h \right\|_2 - \left\| b_h^{\prime} \right\|_2 \right|,
\end{align}
where $L_h=(H-h+1)L$ and $L=\max l_h$.
\end{Lemma}

\begin{proof}
By Lemma \ref{WUV censored5}, we know that $\underline{Q}_h$ is generalized $L_h$-Lipschitz:
\[
\left|\underline{Q}_{h}\left( b_h\right)-\underline{Q}_{h}\left(b_h^{\prime}\right)\right|
\leq L_h\left| \left\| b_h \right\|_2 - \left\| b_h^{\prime} \right\|_2 \right|.
\]

Recall that $\widetilde{V}_h$ is defined as
\[
\widetilde{V}_h(s_h) = \int_{a_h \in \mathcal{A}} \underline{Q}_h(s_h, a_h) \, \widetilde{\pi}_h(a_h \mid s_h) \, da_h,
\]
where the integral is taken over the continuous action space $\mathcal{A}$.

Then for any $b_h = (s_h, a_h)$ and $b_h' = (s_h', a_h')$, we have
\[
\begin{aligned}
&\left|\widetilde{V}_h(b_h) - \widetilde{V}_h(b_h')\right|\\
&= \left| \int \underline{Q}_h(s_h, a_h) \, \widetilde{\pi}_h(a_h \mid s_h) \, da_h
- \int \underline{Q}_h(s_h', a_h) \, \widetilde{\pi}_h(a_h \mid s_h') \, da_h \right| \\
&\leq \int \left| \underline{Q}_h(s_h, a_h) - \underline{Q}_h(s_h', a_h) \right| \, \widetilde{\pi}_h(a_h \mid s_h) \, da_h \\
&\leq L_h \int \left\| b_h - b_h' \right\|_2 \, \widetilde{\pi}_h(a_h \mid s_h) \, da_h \\
&= L_h \left\| b_h - b_h' \right\|_2.
\end{aligned}
\]

Here we used that $\widetilde{\pi}_h$ is a probability density function (integrates to 1), so the generalized Lipschitz constant is preserved under the expectation (integral).

This completes the proof.
\end{proof}

\subsubsection{Proof of Theorem 1}

\begin{proof}
We follow the argument in the proof of Lemma~\ref{WUV censored3}. For any $(s_h, a_h) \in [0, \boldsymbol{\lambda}_h]$, suppose it lies in zone $m$ and is represented as a convex combination:
\[
(s_h, a_h) = \sum_{j=1}^{w+d} m_j b_{m}^j,
\]
where $\{b_m^j\}_{j=1}^{w+d}$ are basis vectors in zone $m$ and $\sum_j m_j = 1$.

By definition of $\underline{Q}_h$, we have:
\[
\underline{Q}_h(s_h, a_h) = \sum_{j=1}^{w+d} m_j \overline{Q}_h(b_m^j).
\]

Then,
\begin{align*}
&\left|\underline{Q}_{h}(s_h, a_h) - \overline{Q}_{h}(s_h, a_h)\right|\\
&= \left|\sum_{j=1}^{w+d} m_j \overline{Q}_h(b_m^j) - \overline{Q}_{h}(s_h, a_h)\right| \\
&\leq \sum_{j=1}^{w+d} m_j \left| \overline{Q}_h(b_m^j) - \overline{Q}_{h}(s_h, a_h) \right| \quad \text{(triangle inequality)} \\
&\leq \sum_{j=1}^{w+d} m_j \cdot \frac{L_h \sqrt{w+d}}{M} \quad \text{(by Lemma~\ref{WUV censored3})} \\
&= \frac{L_h \sqrt{w+d}}{M}.
\end{align*}

For the case where $(s_h, a_h) \in [\boldsymbol{\lambda}_h, \mathbf{I}]$, we have:
\begin{align*}
\left|\underline{Q}_{h}(s_h, a_h) - \overline{Q}_{h}(s_h, a_h)\right|
&= \left|\overline{Q}_h(\boldsymbol{\lambda}_h) - \overline{Q}_{h}(s_h, a_h)\right|\\&
\leq L_h \left|w + d - \|\boldsymbol{\lambda}_h\|_2\right|.
\end{align*}

Combining both cases completes the proof.
\end{proof}
\subsection{Proof of the Sub-optimality Bound}
\begin{Assumption}
Define the marginal state-action occupancy for policy $\pi$ at time $h$:
\(d_h^\pi(s, a) := \mathbb{P}[s_h = s \mid s_1 \sim d_1, \pi] \cdot \pi_h(a \mid s).\) There exists an optimal policy $\pi^{\star}$ such that $\pi^{\star}$ is fully supported by a behavior policy $\mu$, i.e., for all $(s_h, a_h) \in \mathcal{S} \times \mathcal{A}$,
\(d_h^\mu(s_h, a_h) > 0\) if \( d_h^{\pi^\star}(s_h, a_h) > 0. \)
 
\end{Assumption}

This assumption is minimal, requiring only that the behavior policy $\mu$ sufficiently covers the state–action space visited by at least one optimal policy. Such a condition is standard in many RL works~\citep{qiao2024offline,yin2021towards}, and it guarantees the learnability of the optimal value $v^{\star}$. In particular, it represents the minimal requirement necessary for accurately estimating $v^{\star}$, as it ensures that the observed trajectories under $\mu$ provide sufficient coverage of optimal behavior.

\begin{Lemma}[ Theorem 3.4.\citep{qiao2024offline}]
\label{lemma: MDP}Define the \emph{trackable set} as
\( \mathcal{C}_h := \left\{(s_h, a_h) : d_h^\mu(s_h, a_h) > 0 \right\}.\) 
DP-APVI satisfies $\rho$-zCDP. Furthermore, under Assumption 1, denote $\bar{d}_m:=\min _{h \in[H]}\left\{d_h^\mu\left(s_h, a_h\right): d_h^\mu\left(s_h, a_h\right)>0\right\}$. For any $0<\delta<1$, there exists constant $c_1>0$, such that when $n>c_1 \cdot \max \left\{H^2, E_\rho\right\} / \bar{d}_m \cdot \iota(\iota=\log (H S A / \delta))$, with probability $1-\delta$, the output policy $\widetilde{\pi}$ of DP-APVI satisfies
$$0 \leq v^{\star}-v^{\widehat{\pi}}$$
$$
 \leq 4 \sqrt{2} \sum_{h=1}^H \sum_{\left(s_h, a_h\right) \in \mathcal{C}_h} d_h^{\pi^{\star}}\left(s_h, a_h\right) \sqrt{\frac{\operatorname{Var}_{P_h\left(\cdot \mid s_h, a_h\right)}\left(V_{h+1}^{\star}(\cdot)\right) \cdot \iota}{n d_h^\mu\left(s_h, a_h\right)}}
$$
$$+\tilde{O}\left(\frac{H^3+S H^2 E_\rho}{n \cdot \bar{d}_m}\right)$$
where $\tilde{O}$ hides constants and Polylog terms, $E_\rho=4 \sqrt{\frac{H \log \frac{4 H S^2 A}{\delta}}{\rho}}$.
\end{Lemma}

Next, we show the sub-optimality bound of  Algorithm 1 below.

\begin{Theorem}[Optimality Gap]
\label{thm1}
 under Assumption 1, denote $\bar{d}_m:=\min _{h \in[H]}\left\{d_h^\mu\left(s_h, a_h\right): d_h^\mu\left(s_h, a_h\right)>0\right\}$. For any $0<\delta<1$, there exists constant $c_1>0$, such that when $n>c_1 \cdot \max \left\{H^2, E_\rho\right\} / \bar{d}_m \cdot \iota(\iota=\log (H M(w+d) / \delta))$, with probability $1-\delta$,  the output policy $\widetilde{\pi}$ of Algorithm 1 satisfies
\begin{align*}
   & 0 \leq v^{\star}-v^{\widetilde{\pi}}\\& \leq 4 \sqrt{2} \sum_{h=1}^H \sum_{\left(s_h, a_h\right) \in \mathcal{C}_h} d_h^{\pi^{\star}}\left(s_h, a_h\right) \sqrt{\frac{\operatorname{Var}_{P_h\left(\cdot \mid s_h, a_h\right)}\left(V_{h+1}^{\star}(\cdot)\right) \cdot \iota}{n d_h^\mu\left(s_h, a_h\right)}}\\
    &\quad +\tilde{O}\left(\frac{H^3+Mw H^2 E_\rho}{n \cdot \bar{d}_m}\right)\\&+ \frac{ L_1 \sqrt{w + d}}{M}  + L_1 \left|w + d - \|\boldsymbol{\lambda}_h\|_2\right|
\end{align*}
where $\tilde{O}$ hides constants and Polylog terms, $E_\rho=4 \sqrt{\frac{H \log \frac{4 H M^2w(w+d)}{\delta}}{\rho}}$.
\end{Theorem}

\begin{proof}
 We discretize the state and action spaces into $M$ zones, each containing $(w + d)$ basis vectors, in order to adapt the DP-APVI algorithm  \citep{qiao2024offline} to the multi-dimensional RL setting with one-sided feedback. According to their Lemma~\ref{lemma: MDP}, the sub-optimality gap induced by estimation is bounded as:
\begin{align*}
&\text{Sub-Optimality Gap}_{\text{MDP}} 
\\&= 4\sqrt{2} \sum_{h=1}^H \sum_{(s_h, a_h) \in \mathcal{C}_h} d_h^{\pi^\star}(s_h, a_h) 
\sqrt{ \frac{ \operatorname{Var}_{P_h(\cdot \mid s_h, a_h)}\left(V_{h+1}^\star(\cdot)\right) \cdot \iota }{ n \cdot d_h^\mu(s_h, a_h) } }\\
&\quad
+ \tilde{O}\left( \frac{H^3 + M w H^2 E_\rho}{n \cdot \bar{d}_m} \right).
\end{align*}

In addition to this estimation error, discretization itself incurs an approximation error. By Theorem 1, the maximum discrepancy between the approximated and true $Q$-values due to discretization is bounded by:
\[
\frac{L_1 \sqrt{w + d}}{M} + L_1 \left| w + d - \|\boldsymbol{\lambda}_h\|_2 \right|.
\]

Therefore, the total optimality gap is bounded by:
\begin{align*}
&4\sqrt{2} \sum_{h=1}^H \sum_{(s_h, a_h) \in \mathcal{C}_h} d_h^{\pi^\star}(s_h, a_h) 
\sqrt{ \frac{ \operatorname{Var}_{P_h(\cdot \mid s_h, a_h)}\left(V_{h+1}^\star(\cdot)\right) \cdot \iota }{ n \cdot d_h^\mu(s_h, a_h) } } \\
&\quad + \tilde{O}\left( \frac{H^3 + M w H^2 E_\rho}{n \cdot \bar{d}_m} \right) 
+ \frac{L_1 \sqrt{w + d}}{M} 
+ L_1 \left| w + d - \|\boldsymbol{\lambda}_h\|_2 \right|.
\end{align*}

This completes the proof.
\end{proof}

\subsection{Sample Complexity: Theorem 2}
\begin{proof}
By Theorem \ref{thm1}, we have:
\[
\begin{aligned}
&v^{\pi^{\star}} - v^{\widetilde{\pi}} \\&\leq 4\sqrt{2} \sum_{h=1}^H \sum_{(s_h, a_h) \in \mathcal{C}_h} d_h^{\pi^{\star}}(s_h, a_h) \cdot \sqrt{ \frac{ \operatorname{Var}_{P_h(\cdot \mid s_h, a_h)}\left(V_{h+1}^{\star}(\cdot)\right) \cdot \iota }{ n \cdot d_h^\mu(s_h, a_h) } } \\
&\quad + \tilde{O}\left( \frac{H^3 + M w H^2 E_\rho}{n \cdot \bar{d}_m} \right) + \frac{ L_1 \sqrt{w + d}}{M}  + L_1 \left|w + d - \|\boldsymbol{\lambda}_h\|_2\right|.
\end{aligned}
\]
To ensure that the learned policy $\widetilde{\pi}$ satisfies \( v^{\star} - \alpha - \tilde{O}\left(L_1 \left|w + d - \|\boldsymbol{\lambda}_h\|_2\right|\right) \leq v^{\widetilde{\pi}} \leq v^{\star}\), it suffices to set the number of discretization zones $M$ and sample size $n$ such that:
\[
M \geq M(L_1, \alpha, d, w) = O\left(\frac{ L_1 \sqrt{w + d} }{ \alpha } \right), 
\]
$$
n \geq n(H, E_\rho, \alpha) = \tilde{O}\left( \frac{(1+ E_\rho) H^3 }{ \alpha^2 } \right).$$

This completes the proof.
\end{proof}

\bibliographystyle{named}
\bibliography{ijcai26}

\end{document}